\documentclass{article}

\usepackage{microtype}
\usepackage{graphicx}
\usepackage{subfigure}
\usepackage{booktabs} 

\usepackage{hyperref}
\usepackage{url}

\usepackage[accepted]{icml2025}

\usepackage{amsmath}
\usepackage{amssymb}
\usepackage{mathtools}
\usepackage{amsthm}

\theoremstyle{plain}
\newtheorem{theorem}{Theorem}[section]
\newtheorem*{theorem*}{Theorem}

\newtheorem{lemma}[theorem]{Lemma}
\newtheorem*{lemma*}{Lemma}
\newtheorem{defn}[theorem]{Definition}

\newtheorem{remark}[theorem]{Remark}

\newtheorem{assumptionA}{Assumption}

\newtheorem{assumptionB}{Assumption}

\providecommand{\norm}[1]{\left\lVert#1\right\rVert}

\renewcommand{\hat}{\widehat}

\newcommand{\bfm}[1]{\ensuremath{\mathbf{#1}}}

\def\ba{\bfm a}   \def\bA{\bfm A}   
   \def\bB{\bfm B}   
      
      \def\cD{{\cal  D}}
      
\def\bff{\bfm f}     
   \def\bG{\bfm G}   
   \def\bH{\bfm H}   
     \def\II{\mathbb{I}} \def\cI{{\cal  I}}

\def\\bl{\bfm l}      
   \def\bM{\bfm M}   
      \def\cN{{\cal  N}}

     \def\PP{\mathbb{P}} 
      \def\cQ{{\cal  Q}}
     \def\RR{\mathbb{R}}

   \def\bU{\bfm U}   
   \def\bV{\bfm V}   
      
\def\bx{\bfm x}   \def\bX{\bfm X}   
      
\def\bz{\bfm z}      

\def\bzero{\bfm 0}

\def\polylog{{\rm PolyLog}}

\def\dl{\dot{\ell}}
\def\ddl{\ddot{\ell}}
\def\dddl{\dddot{\ell}}

\def\ddr{\ddot{r}}

\newcommand{\bfsym}[1]{\ensuremath{\boldsymbol{#1}}}

\def\bbeta{\bfsym \beta}

             \def\bSigma{\bfsym \Sigma}

                \def\hbbeta{\hat{\bfsym \beta}} 
\def\hb{\hbbeta}

\def\hbi{\hbbeta_{/i}}

\def\tbi{\tilde{\bbeta}_{/i}}

\def\tbin{\tbi^{\mathrm{Newton}}}
\def\cItrue{\cI^{\mathrm{True}}}
\def\cIif{\cI^{\mathrm{IF}}}
\def\cIn{\cI^{\mathrm{New}}}

\DeclareMathOperator{\argmin}{argmin}

\DeclareMathOperator{\diag}{diag}

\DeclareMathOperator{\tr}{tr}
\newcommand\numberthis{\addtocounter{equation}{1}\tag{\theequation}}

\icmltitlerunning{Influence in High Dimensions}

\begin{document}

\twocolumn[
\icmltitle{Newfluence: Boosting Model interpretability and Understanding in High Dimensions}




\begin{icmlauthorlist}
\icmlauthor{Haolin Zou}{a}
\icmlauthor{Arnab Auddy}{b}
\icmlauthor{Yongchan Kwon}{a}
\icmlauthor{Kamiar Rahnama Rad}{c}
\icmlauthor{Arian Maleki}{a}
\end{icmlauthorlist}
\icmlaffiliation{a}{Department of Statistics, Columbia University, New York, NY, US.}
\icmlaffiliation{b}{Department of Statistics, The Ohio State University, Columbus, OH, US.}
\icmlaffiliation{c}{Baruch College, The City University of New York, New York, NY, US.}

\icmlcorrespondingauthor{Haolin Zou}{hz2574@columbia.edu}


\vskip 0.3in
]



\printAffiliationsAndNotice{}  

\begin{abstract}

The increasing complexity of machine learning (ML) and artificial intelligence (AI) models has created a pressing need for tools that help scientists, engineers, and policymakers interpret and refine model decisions and predictions. Influence functions, originating from robust statistics, have emerged as a popular approach for this purpose.

However, the heuristic foundations of influence functions rely on low-dimensional assumptions where the number of parameters $p$ is much smaller than the number of observations $n$. In contrast, modern AI models often operate in high-dimensional regimes with large $p$, challenging these assumptions.

In this paper, we examine the accuracy of influence functions in high-dimensional settings. Our theoretical and empirical analyses reveal that influence functions cannot reliably fulfill their intended purpose. We then introduce an alternative approximation, called Newfluence, that maintains similar computational efficiency while offering significantly improved accuracy. 

Newfluence is expected to provide more accurate insights than many existing methods for interpreting complex AI models and diagnosing their issues. Moreover, the high-dimensional framework we develop in this paper can also be applied to analyze other popular techniques, such as Shapley values.
\end{abstract}

\section{Introduction }
\subsection{Background and literature review}\label{ssec:background}

The growing complexity and black-box nature of machine learning (ML) and artificial intelligence (AI) models have made their assessment and interpretation critical challenges, especially when it comes to informed decision-making. Attribution-based techniques, such as influence functions (IF) \citet{pmlr-v70-koh17a, han2020explaining, pruthi2020estimating, yeh2019fidelity, hammoudeh2024training, park2023trak} and Shapley values \cite{ ghorbani2019data, jia2019towards, sundararajan2020many, rozemberczki2022shapley, kwon2022beta, wang2024data}, have emerged as widely used tools for understanding model behavior. 

One of the most widely used attribution-based techniques relies on the IF, a concept originally developed in the field of robust statistics \cite{hampel1974influence}. IFs quantify the effect of small perturbations to individual data points on the predictions of an ML or AI model. IFs have demonstrated promising results in a variety of downstream tasks, including interpreting model predictions \citep{ilyas2022datamodels, grosse2023studying, kwon2024datainf}, improving model alignment \citep{zhang2025correcting, min2025understanding}, and analyzing training dynamics \citep{guu2023simfluence, wang2024capturing}. 

To understand some of the challenges faced by IFs, consider the dataset $\cD = \{\bz_1, \bz_2, \ldots, \bz_n\}$ being used for learning the parameters $\bbeta \in \mathbb{R}^p$ of an AI model. Also assume that for estimating $\bbeta$ we use the empirical risk minimization:
\[
    \hbbeta = \argmin_{\bbeta} L_n(\bbeta):= \sum_{j=1}^{n} \ell(\bbeta, \bz_j),
\]
where  $\ell(\bbeta, \bz)$ denotes the loss function. Using $\hbbeta$ we can make predictions about a new data sample $\bz_0$ and the accuracy of our prediction is measured as $\ell(\hbbeta, \bz_0)$. Hence, we measure the influence of the datapoint $\bz_i$ on the prediction of our model using: 
\begin{equation}\label{eq:trueinf}
    \cI^{\text{True}}(\bz_i,\bz_0):=
    \ell(\hbbeta_{/i},\bz_0)
    -\ell(\hbbeta,\bz_0),
\end{equation}
where $\hbbeta_{/i} := \argmin_{\bbeta} L_{n,/i}(\bbeta):= \sum_{j\neq i}^{n} \ell(\bbeta, \bz_j)$. We call this quantity the \textbf{true} influence of $\bz_i$. 

Calculating $\cI^{\text{True}}(\bz_i,\bz_0)$ requires retraining the model for calculating $\hbbeta_{/i}$. Since this is computationally demanding \citet{pmlr-v70-koh17a} proposed the following approximations. First, using the first-order Taylor approximation, we have
\begin{eqnarray}\label{eq:influence:firstAp}
   \cI^{\text{True}}(\bz_i,\bz_0)
   \approx \nabla_{\bbeta} \ell(\hbbeta, \bz_0)^\top (\hbbeta_{/i} -\hbbeta ) 
\end{eqnarray}
 The second step of approximation is used to calculate $(\hbbeta_{/i} -\hbbeta )$ efficiently. Defining:
 \[
 \widetilde{\bbeta}(\epsilon) := \sum_{j=1}^{n} \ell(\bbeta, \bz_j) - \epsilon \ell(\beta, \bz_i)
 \]
 Note that $\widetilde{\bbeta}(1) = \hbbeta_{/i}$. Since $n$ is a large number, we have:
\begin{equation}\label{eq:influence:secondAp}
\hbbeta_{/i} -\hbbeta  \approx -\left. \frac{d\widetilde{\bbeta}(\epsilon)}{d \epsilon} \right|_{\epsilon=0}. 
\end{equation}
The derivative $\left. \frac{d\widetilde{\bbeta}(\epsilon)}{d \epsilon} \right|_{\epsilon=0}$ can be calculated using the Hessian  of the empirical risk, i.e. $\bG =   \sum_{i=1}^{n} \nabla^2_{\bbeta} \ell(\bbeta, \bz_i)$: 
\begin{equation}\label{eq:hessiancalc}
\frac{d\hbbeta}{d\epsilon} \Big|_{\epsilon=0} = -\bG^{-1} \nabla_{\bbeta} \ell(\hbbeta, \bz_i).
\end{equation}
Combining \eqref{eq:influence:firstAp} and \eqref{eq:hessiancalc}, \citet{pmlr-v70-koh17a} proposed the following approximation for $\cI^{\text{True}}(\bz_i,\bz_0)$:

\begin{equation}\label{eq:influence_final:formula}
\mathcal{I}^{\mathrm{IF}}(\bz_i, \bz_0) =  \nabla_{\bbeta} \ell(\hbbeta, \bz_0)^\top \bG^{-1} \nabla_{\bbeta} \ell(\hbbeta, \bz_i).
\end{equation}

Inspired by the analysis offered in \citet{hampel1974influence} for low dimensional settings, i.e. the setting, where the number of parameters $p$ is much smaller than the number of observations $n$ ($p \ll n$), it is widely believed that the approximations we mentioned in \eqref{eq:influence:firstAp} and \eqref{eq:influence:secondAp} are accurate. However, many modern AI and ML models have many parameters, that challenges the assumption $p\ll n$. Throughout the paper, we call the models in which $p$ is not much smaller than $n$, high-dimensional models. Inspired by such models we would like to answer the following question:

$\mathbf{\mathcal{Q}}_1$: Are the two approximations that led to \eqref{eq:influence_final:formula} accurate under the high-dimensional settings? 

$\mathbf{\mathcal{Q}}_2$: If the answer to $\cQ_1$ is negative, can the formula presented in \eqref{eq:influence_final:formula}  be improved to yield an accurate approximation for $\cI^{\text{True}}(\bz_i,\bz_0)$? 

A few empirical papers have reported the inaccuracies in the conclusions of the IF; 
\citet{pmlr-v70-koh17a} and \citet{bae2022if} empirically showed that IFs are often inaccurate in estimating leave-one-out scores, particularly when applied to deep neural network models.
\citet{basu2020influence} studied how IFs change across different model parameterizations and regularization techniques, showing they can be erroneous in some circumstances. \citet{schioppa2023theoretical} examined the five major sources of inaccuracy and highlighted the potential pitfalls of the Taylor expansion that is commonly used in IFs.

Our goal is to develop a theoretical framework for analyzing the accuracy of IFs. We argue that the high dimensionality of modern AI and ML models undermines the approximations on which IFs are based. To support this claim, we adopt a high-dimensional asymptotic regime where $n, p \to \infty$ with $n/p \to \gamma$, for some fixed $\gamma$ \cite{donoho2011noise, zheng2017does, donoho2016high, el2013robust, sur2019likelihood, li2021minimum}. That is, both $n$ and $p$ are large, but their ratio remains bounded. Our theoretical results show that the answer to $\mathbf{\mathcal{Q}}_1$ is negative, i.e. IFs can be inaccurate in high-dimensional settings. In response, we propose an alternative approximation of $\cI^{\text{True}}(\bz_i, \bz_0)$, called Newfluence, that retains the computational simplicity of classical IFs but remain accurate under high-dimensional conditions. Empirical results further support our theoretical findings. Figure \ref{fig:ALO_IF_scatter_plt} compares the performance of Newfluence with that of $\mathcal{I}^{\mathrm{IF}}(\bz_i, \bz_0)$. 
\begin{figure}[t!]
    \centering
       \includegraphics[width=0.45\textwidth]{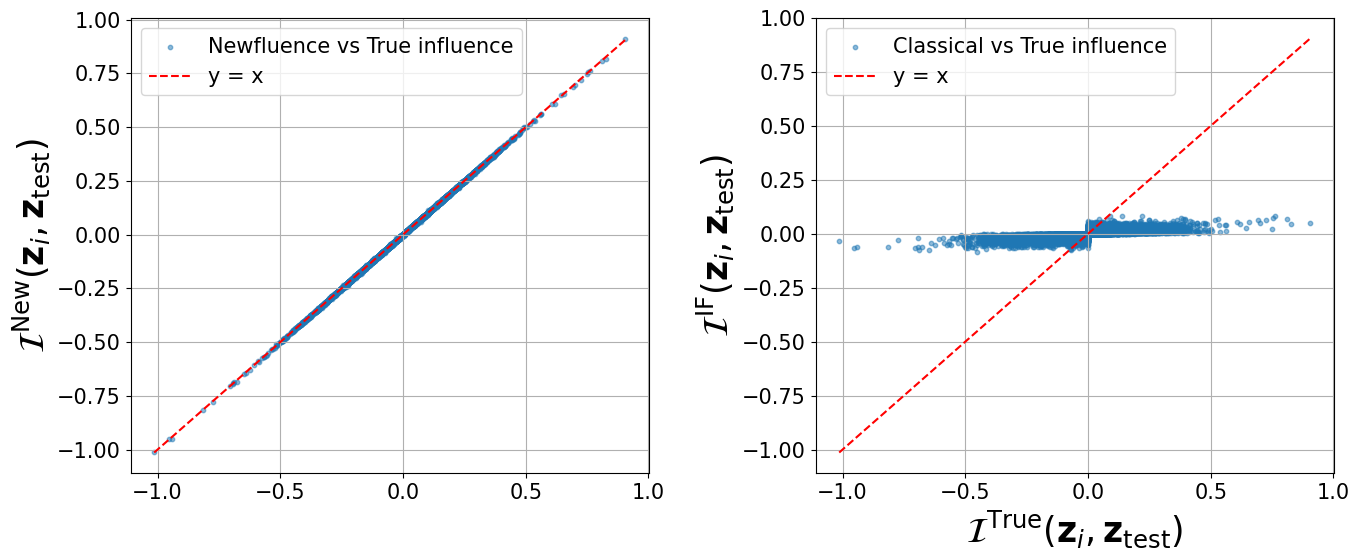}
    \caption{Comparison of Newfluence and $\mathcal{I}^{\mathrm{IF}}(\bz_i, \bz_0)$. We use logistic ridge regression with $n = 500$, $p = 1000$, and $\lambda = 0.01$. The other details of the simulation are presented in Section \ref{sec:exp} The figure shows results for all the influence of all the $n=500$ training points on the prediction loss of $m = 100$ unseen new test points generated from the true logistic model. \textbf{Left:} Newfluence vs. true  influence. \textbf{Right:} $\mathcal{I}^{\mathrm{IF}}(\bz_i, \bz_0)$ vs. true influence. The Newfluence-based method offers a significantly better approximation.}
    \label{fig:ALO_IF_scatter_plt}
\end{figure}
\subsection{Notations}
\label{ssec: notations}
Vectors and matrices are represented with boldfaced lower and upper case letters respectively, e.g. $\ba\in\RR^n,\bX\in\RR^{n\times p}$. 
For matrix $\bX$, $\sigma_{\min}(\bX)$, $\sigma_{\max}(\bX)$ , $\tr(\bX)$ denote its minimum and maximum singular values and trace respectively. 
We denote any polynomials of $\log(n)$ by $\polylog(n)$.

We use classic notations for deterministic and stochastic limit symbols such as $o(\cdot),O(\cdot),o_P(\cdot)$ and $O_P(\cdot)$. In addition, we use the notation $X=\Theta_P(1)$ if $X= O_P(1)$ but not $o_P(1)$, and $X=\Theta_P(a_n)$ if and onlf if $X/a_n=\Theta_P(1)$.


\section{Theoretical results}
\label{sec: theo}
\subsection{Newfluence for High dimensional R-ERM}
\label{ssec: problem}
In this section, we formalize our ideas using the generalized linear model and propose our new measure of influence, called \textbf{Newfluence}. Consider the generalized linear model: $\bz_i = (y_i, \bx_i)$, where $y_i\in\RR$ is the response and $\bx_i\in\RR^p$ is the feature. We further assume that a given dataset $D_n:=\{(y_i,\bx_i)\}_{i=1}^n$ consists of $n$ i.i.d. observations from a generalized linear model, i.e. $(y,\bx)\sim p(\bx)q(y|\bx^\top\bbeta^*)$ , and $\bbeta^*$ represents the parameters that the ML system needs to learn.  

We use the following regularized empirical risk minimization (R-ERM) for estimate $\bbeta^*$:
\vspace{-0.3cm}
\[
    \hbbeta := \underset{\bbeta\in \RR^p}{\argmin} L_n(\bbeta) =\sum_{j=1}^n \ell(y_j,\bx_j^\top\bbeta) + \lambda r(\bbeta)
    \label{eq:hbbeta}\numberthis,
\]
where with a slight abuse of notation, we have redefined the loss function $\ell(y,u)$ as a function of $y$ and $u = \bx_j^\top \bbeta$. Examples of the loss include square loss $\frac12 (y-u)^2$, or negative log-likelihood $-\log q(y|\bx^\top\bbeta=u)$. Furthermore, $r: \RR^p \to \RR$ is the regularizer, e.g. LASSO: $r(\bbeta) = \norm{\bbeta}_1$, ridge: $r(\bbeta) = \norm{\bbeta}_2^2$, and $\lambda>0$ is the strength of regularization. In the rest of the paper, we use simplified notations: for $j \in \{0,1,\dots,n\}$, $\ell_j(\bbeta):=\ell(y_j,\bx_j^\top\bbeta)$, $\dl_j(\bbeta) = \left.\frac{\partial}{\partial u}\ell(y_j,u)\right|_{u=\bx_j^\top\bbeta}$ and $\ddl_j(\bbeta) = \left.\frac{\partial^2}{\partial u^2}\ell(y_j,u)\right|_{u=\bx_j^\top\bbeta}$.
Rewriting \eqref{eq:trueinf} and \eqref{eq:influence_final:formula} under the setting of this section, we obtain:
\[
    \cItrue(\bz_i,\bz_0):=\ell_0(\hbi) - \ell_0(\hb),
\]
\[
    \cIif(\bz_i,\bz_0):=\dl_0(\hb)\bx_0^\top\bG^{-1}(\hb)\bx_i\dl_i(\hb). 
\]
In response to $\mathcal{Q}_1$ in Section \ref{ssec:background}, we study the error $|\cIif(\bz_i,\bz_0) - \cItrue(\bz_i,\bz_0)|$ under the asymptotic setting $n,p \rightarrow \infty$, while $n/p \rightarrow \gamma_0$, where $\gamma_0$ is a fixed (but arbitrary) number. Note that the assumption $n/p \rightarrow \gamma_0$ aims to cover high-dimensional problems. For instance, but choosing $\gamma<1$, we it will even cover the situation where the number of features are less than the number of observations.  

The classical expectation was for the IF approximation $\cIif(\bz_i,\bz_0)$ to closely match the true influence $\cItrue(\bz_i,\bz_0)$, making the difference negligible. However, our theoretical results in the next section show that, contrary to this expectation,
\[
\cIif(\bz_i, \bz_0)=(1-H_{ii})\cItrue (\bz_i, \bz_0)+O_p\left(\frac{\polylog(n)}{n}\right), 
\]
where $H_{ii}=(\bx_i^{\top}\bG^{-1}\bx_i)\ddl_i(\hb)$. As will be clarified later, it follows from the definition of $\bG$ and our high-dimensional framework that $H_{ii}=\Theta_P(1)$ and $0\leq H_{ii} \leq 1$.


This implies that the commonly used $\cIif$ formula \emph{underestimates} the true influence by a datapoint-dependent factor $(1 - H_{ii})$.  Since $1 - H_{ii}$ varies with the datapoint $\bz_i$, it is conceivable that for points with large true influence $\cItrue(\bz_i, \bz_0)$, the value of $1 - H_{ii}$ may be small—causing $\cIif(\bz_i, \bz_0)$ to be markedly lower and incorrectly suggesting that $\bz_i$ is non-influential in the prediction.This underestimation phenomenon is also illustrated in Figure~\ref{fig:ALO_IF_scatter_plt}.

This motivates $\mathcal{Q}_2$ in Section \ref{ssec:background}: can we find a more accurate alternative? In response, we propose a modified influence function framework that addresses this limitation while retaining computational efficiency. First, note that the above calculation suggests the corrected estimator:
\[
\Tilde{\cI}^{\rm IF}
:=\frac{1}{1-H_{ii}}\cIif
\]
which is in fact consistent for $\cItrue$. However, both $\cIif$ and $\Tilde{\cI}^{\rm IF}$ require computing the gradient $\nabla_{\bbeta} \ell(\hbbeta, \bz_0)$ for each new test point $\bz_0$. We avoid this in our proposal called \textbf{Newfluence}, which is given by
\begin{equation}
    \cIn(\bz_i, \bz_0) =  \ell_0\left(\hb +  \frac{       \dl_i(\hb)\bG^{-1}\bx_i}{1-H_{ii}}\right)
    -\ell_0\left(\hb\right).
    \label{eq: def_newfluence}
\end{equation}
Evaluating the loss function at any $\bbeta$ and $\bz_0$ should have similar computational demand as evaluating the gradient, so this method generally does not increase computational complexity. As we shall show in the next section, this estimator is consistent for the true influence $\cItrue$.
%


This alternative approach retains the interpretability advantages of IFs while eliminating unnecessary approximations and improving computational accuracy. This formula is inspired by the recent work on the literature of risk estimation in high-dimensional settings \cite{rad2018scalable, rad2020error, auddy24a, wang2018approximate}. We present the derivation of the formula in Appendix \ref{sec: derive_cIn}. 


\subsection{Our main theoretical contributions: Smooth case}
In this section, we formally state our main theoretical results. Before stating our results, we first review some of the assumptions we have made in our analysis. All these assumptions are mild, and have been shown to be satisfied by a large number of models \cite{rad2018scalable, auddy24a, el2013robust, sur2019likelihood}. 

\begin{assumptionA}\label{assum:separability}
    The regularizer is separable:
    \[
        r(\bbeta) = \sum_{k=1}^pr_k(\beta_k).
    \]
\end{assumptionA}
\begin{assumptionA}\label{assum:smoothness}
Both the loss function $\ell:\RR\times\RR\to\RR_+$ and the regularizer $r:\RR^p\to\RR_+$ are twice differentiable. 
\end{assumptionA}

\begin{assumptionA}\label{assum:convexity}
Both $\ell$ and $r$ are proper convex, and $r$ is $\nu$-strongly convex in $\bbeta$ for some constant $\nu>0$.
\end{assumptionA}

\begin{assumptionA}\label{assum:ld}
    $\exists C,s>0$ such that 
    \[
        \max\{\ell(y,u), |\dl(y,z)|,|\dddl(y,z)| \}\leq C(1+|y|^s + |z|^s)
    \]
    and that $\nabla^2 r(\bbeta) = \diag[\ddr_k(\beta_k)]_{k\in[p]}$ is $C_{rr}(n)$-Lipschitz (in Frobenius norm) in $\bbeta$ for some $C_{rr}(n)=O(\polylog(n))$.
\end{assumptionA}
We also adopt the following assumptions on the data generation mechanism:

\begin{assumptionB}\label{assum:normality}
The feature vectors $\bx_i \overset{iid}{\sim} \mathcal{N} (\bzero, \bSigma)$. Furthermore, $ \lambda_{\max}(\bSigma) \leq \frac{C_X}{p}$, for some constant $C_X>0$. 
\end{assumptionB}

\begin{assumptionB}\label{assum:y}
    $\PP(|y_i|>C_y(n))\leq q_n^{(y)}$ for some $C_y(n)=O(\polylog(n))$ and $q_n^{(y)}=o(n^{-1})$. 
\end{assumptionB}

\begin{theorem}\label{thm: main_smooth} Under assumptions \ref{assum:separability},... \ref{assum:ld}, and \ref{assum:normality},\ref{assum:y},  we have 

1. $|\cIn-\cItrue|=o_P(\frac{1}{n}\polylog(n))$

2. $|\cItrue|=O_P(\frac{1}{\sqrt{n}}\polylog(n)),$

3.  $\cIif=(1-H_{ii})\cItrue+o_P(\frac{1}{n}\polylog(n))$.
\end{theorem}
We present 
the proof 
in Appendix \ref{sec: heuristic}. 
We now present a few remarks on the above theorem.


\begin{remark}
    For many common choices of loss functions and regularization parameters, $\ddl_i(\hbbeta)=\Theta_P(1)$, and $\max_j\Ddot{r}_j(\hbbeta_j)=O_p(1)$. Then it follows from the definition of $H_{ii}$ and our assumptions that $H_{ii}=\Theta_P(1/(1+\lambda))$. Thus the classical estimator $\cIif$ incurs considerable bias in such situations, when $\lambda$ is not very large. This is reflected in our simulation experiments in the next section.
\end{remark}

\begin{remark}
Comparing $\cItrue$ and $\cIif $, we find that the approximation error in $\cIif$ is $H_{ii}\cdot\cItrue$, which, for moderate regularization strength $\lambda$, is of the same order as the true influence itself. As a result, a highly influential data point may appear non-influential due to approximation errors involved in obtaining $|\cIif|$, potentially leading to misleading conclusions—an issue also noted in prior empirical studies \citep{basu2020influence, bae2022if}.
\end{remark}

\begin{remark}
In contrast to $\cIif$, the error $|\cIn - \cItrue|$ is much smaller than $|\cItrue|$, indicating that $\cIn$ provides a reliable approximation of $\cItrue$ when $n$ and $p$ are large.
\end{remark}


\section{Numerical Experiments}\label{sec:exp}

We evaluate the accuracy of Newfluence and classical influence function in $\ell_2$-regularized logistic regression. We generate synthetic binary classification datasets with feature dimensions $p \in \{500, 1000, 2000\}$ and set the sample size to maintain a fixed ratio $n/p = 0.5$. The true model coefficients $\bbeta^\star \sim \mathcal{N}(\mathbf{0}, \mathbf{I}_p)$ are drawn from a standard normal distribution, and input features are sampled as $\bx \sim \mathcal{N}(\mathbf{0}, \mathbf{I}_p / n)$. Labels $y \in \{0, 1\}^n$ are generated according to a Bernoulli model with success probabilities given by $\sigma(\bx^\top \bbeta^\star)$, where $\sigma$ denotes the sigmoid function. The code used to produce the numerical results is available online\footnote{\url{https://anonymous.4open.science/r/corrected-influence-functions-2F7E/}}.

We fit a logistic ridge model using Newton's method, with regularization parameter $\lambda \in \{0.01, 10\}$. For each model, we compute the true influence function $\cItrue$ for $m = 100$ unseen test examples. These true influence values are compared against two first-order approximations: (i) \textbf{Newfluence}, and (ii) the classical influence function, $\cIif$. For each test point, we evaluate the rank correlation between true and approximated influence across training points using Kendall’s $\tau$.

The value of $\tau$ lies in the interval $[-1, 1]$, with $\tau = 1$ indicating perfect agreement between rankings, $\tau = -1$ indicating perfect disagreement (reversed order), and $\tau = 0$ indicating no correlation in pairwise orderings.

Results are summarized in Tables~\ref{tab:kendall_compact} and~\ref{tab:kendall_compact2}. When $\lambda = 0.01$, corresponding to an effective degrees of freedom ratio $\text{df}/p \approx 0.344$ (where df$:= \sum_{i=1}^n H_{ii}$), Newfluence estimates exhibit nearly perfect rank agreement with the true influence ranking ($\tau \approx 1.00$), while classical IFs are notably less accurate ($\tau \approx 0.88$). In contrast, when $\lambda = 10$, the regularization is strong and the effective degrees of freedom is substantially smaller ($\text{df}/p \approx 0.023$), placing the model in a low-dimensional regime. In this setting, both Newfluence and the classical influence estimates achieve perfect agreement with the true influence ranking.

These results illustrate that classical IFs can be accurate in the low-dimensional regime—i.e., when $\text{df}/p \ll 1$—but fail to match the fidelity of Newfluence in high-dimensional settings where the model complexity (as measured by $\text{df}/p$) is non-negligible.

\section{Concluding Remarks}

Interpreting today’s black-box systems—ranging from foundation models to the latent “world models” that power model-based reinforcement learning—requires attribution tools whose guarantees scale with model dimensionality. In this work we have shown that the classical influence-function approximation of \cite{pmlr-v70-koh17a} incurs a systematic, datapoint-specific bias in high dimensions: it underestimates the true leave-one-out effect by a factor $1-H_{ii}$. To remedy this, we proposed \textsc{Newfluence}, a single-Newton-step correction that preserves the computational economy of influence functions while eliminating their high-dimensional bias. Our theory establishes consistency under the proportional asymptotics $p\asymp n$, and experiments with high-dimensional logistic ridge regression confirm near-perfect rank agreement between \textsc{Newfluence} and ground truth, whereas the classical estimator degrades markedly. 

Although our theoretical analysis presently targets generalized linear models with strongly convex, twice-differentiable objectives, the Newton-step correction 
is model-agnostic. We therefore view \textsc{Newfluence} as a first step toward principled influence estimation in the non-convex, non-smooth, and sequential settings that characterize modern AI models. 


We hope that recognizing and correcting the high-dimensional bias documented here will encourage the community to reassess existing interpretability tools and to design influence-aware debugging, data-valuation, and alignment pipelines that scale with the complexity of contemporary AI systems.

\begin{table}[t]
\centering
\caption{Kendall's $\tau$ (std in parentheses) for Newfluence and classical influence-based estimates across different training set sizes 
computed over  $m=100$ unseen test data points 
for logistic ridge when $\lambda=0.01$, leading to $\text{df}/p=0.344$.}
\label{tab:kendall_compact}
\begin{tabular}{rrrr}
\toprule
$n$ & $p$ & $\tau$ (Newfluence) & $\tau$ (IF) \\
\midrule
250 & 500 & 0.99 (0.00) & 0.88 (0.01) \\
500 & 1000 & 1.00 (0.00) & 0.88 (0.01) \\
1000 & 2000 & 1.00 (0.00) & 0.88 (0.00) \\
\bottomrule
\end{tabular}
\end{table}

\begin{table}[t]
\centering
\caption{Kendall's $\tau$ (std in parentheses) for Newfluence and classical influence-based estimates across different training set sizes 
computed over  $m=100$ unseen test data points 
for logistic ridge when $\lambda=10.00$, leading to $\text{df}/p=0.023$.}
\label{tab:kendall_compact2}
\begin{tabular}{rrrr}
\toprule
$n$ & $p$ & $\tau$ (Newfluence) & $\tau$ (IF) \\
\midrule
250 & 500 & 1.00 (0.00) & 1.00 (0.00) \\
500 & 1000 & 1.00 (0.00) & 1.00 (0.00) \\
1000 & 2000 & 1.00 (0.00) & 1.00 (0.00) \\
\bottomrule
\end{tabular}
\end{table}

\bibliography{reference}
\bibliographystyle{icml2025}

\newpage
\appendix
\onecolumn

\section{Derivation of the Newfluence for smooth problems}\label{sec: derive_cIn}
Here we provide the detailed derivation of $\tbin$. The key is to use one step of Newton method to get an approximation of $\hbi$, and using Woodbury formula to reduce computational complexity of the leave-one-out Hessian.

Recall that the Newton method, or Newton-Raphson method, is an iterative algorithm to find the root of a function. 
\begin{defn}[Newton-Raphson Method]
\label{def: Newton}
    Given a function $\bff: \RR^p\to\RR^p$ with a unique root $\bx^*\in\RR^p$, the Newton-Raphson Method finds $\bx^*$ iteratively, starting from an initial point $\bx^{(0)}$:
    \[
        \bx^{(t)} := \bx^{(t-1)} - \bG^{-1}(\bx^{(t-1)})\bff(\bx^{(t-1)}),
    \]
    where $\bG(\bx)$ is the Jacobian of $f$, which is assumed to exist and invertible.
\end{defn}

Note that when the objective function $L_{n,/i}(\bbeta)$ is smooth, $\hbi$ is the root of its gradient:
\[
    \boldsymbol{0}=\nabla L_{n,/i}(\hbi) = \sum_{j\neq i}\dl_j(\hbi)\bx_j + \lambda \nabla r(\hbi).
\]

If we replace $\bff$ in Definition \ref{def: Newton}, then its Jacobian is just the Hessian of $L_{n,/i}$:
\begin{equation}\label{eq:def-G-minus-i}
        \bG_{/i}(\hb) = \sum_{j\neq i} \bx_j\bx_j^\top\ddl_j(\hb) + \lambda \nabla^2 r(\hb).    
\end{equation}
Moreover, since $L_{n,/i}=L_n - \ell_i$ and $\nabla L_n(\hb)=0$, we have
\[
\nabla L_{n,/i}(\hb) = \nabla L_n(\hb) - \nabla \ell_i(\hb) = - \ell_i(\hb)\bx_i.
\]

Inspired by the corresponding literature (e.g. \cite{rad2018scalable,rad2020error,auddy24a}), we initiate the Newton iteration at the full model parameter $\hb$ and apply \textbf{a single update}, and call the result $\tbin$:
\begin{align*}
    \tbin &:= \hb - \bG_{/i}(\hb)\nabla L_{n,/i}(\hb)\\
    &= \hb + \dl_i(\hb)\bG_{/i}(\hb)\bx_i. 
    \label{eq: def_tbin_appendix}\numberthis
\end{align*}
Furthermore, we use Lemma \ref{lem:woodberry} to simplify the calculation of $\bG_{/i}$ without repeatedly taking inverse for each $i$. For now we write $\bG, \bG_{/i}$ and drop their dependence on $\hb$:
\begin{align*}
    \bG_{/i}^{-1} &= (\bG - \nabla^2 \ell_i(\hb))^{-1}\\
    &= (\bG - \bx_i\bx_i^\top\ddl_i(\hb))^{-1}\\
    (\text{ Woodbury Formula }) \quad 
    &= \bG^{-1} + \frac{\bG^{-1}\bx_i^\top\bx_i\bG^{-1} \ddl_i(\hb)}{1-\bx_i^\top\bG^{-1}\bx_i\ddl_i(\hb)}.
\end{align*}
Inserting this back into \eqref{eq: def_tbin_appendix}:
\begin{align*}\label{eq:beta-newton-woodbury}
    \tbin &= \hb + \dl_i(\hb)\bG_{/i}(\hb)\bx_i\\
    &=   \hb + \dl_i(\hb) 
    \left[
        \bG^{-1} + \frac{\bG^{-1}\bx_i^\top\bx_i\bG^{-1} \ddl_i(\hb)}{1-\bx_i^\top\bG^{-1}\bx_i\ddl_i(\hb)}
    \right]
    \bx_i\\
    &= \hb + \dl_i(\hb)\bG^{-1}\bx_i\frac{1}{1-H_{ii}}\numberthis
\end{align*}
where $H_{ii}$ is the $(i,i)$ element of 
\begin{align*}
    \bH = \bX\bG^{-1}\bX^\top \diag[\ddl_i(\hb)]_{i=1}^n,
\end{align*}
and $\diag[\ddl_i(\hb)]_{i=1}^n$ is the diagonal matrix with diagonal elements being $\{\ddl_i(\hb), i=1,2,...,n\}$.
Replacing $\hbi$ by $\tbin$ in the definition of the true influence $\cItrue$ yields our definition of $\cIn$:
\begin{align*}
    \cIn &= \ell_0(\tbin)-\ell_0(\hb)\\
    &= \ell_0\left(\hb + \dl_i(\hb)\bG_{/i}(\hb)\bx_i\frac{1}{1-H_{ii}}\right)
    -\ell_0\left(\hb\right)
\end{align*}

\subsection{A Heuristic Lower Bound on $H_{ii}$}
We show that $H_{ii}$ is bounded below by a constant with non-vanishing probability in the case of linear regression, i.e., an ERM with squared loss.
\begin{align*}
    H_{ii} &= \bx_i^\top\bG^{-1}(\hb)\bx_i\\
    &\geq \|\bx_i\|^2 \sigma_{\min}(\bG^{-1}(\hb))\\
    &\geq \|\bx_i\|^2 (\sigma_{\max}(\bG(\hb)))^{-1}.
\end{align*}
By standard concentration results for $\chi^2$ distribution (Lemma \ref{lem:chi:sq:ind}), we know that $\PP(\|\bx_i\|^2>\frac12)\geq 1-e^{-\frac{1}{16}p}$. Also, we have
\begin{align*}
    \sigma_{\max}(\bG(\hb))
    =&~ \|\bX^\top\bX + \lambda\nabla^2 r(\bbeta)\|\\
    \leq&~ \|\bX\|^2 + \lambda \|\nabla^2 r(\bbeta)\|.
\end{align*}
 By Lemma \ref{lem:conc_normal_mat}, $\|\bX\|=O_P(1)$. We claim without proof that, for most commonly used regularizers we have $\|\nabla^2 r(\bbeta)\|=O_P(1)$. For example,  for the ridge penalty $r(\bbeta)=\|\bbeta\|^2$ we have $\|\nabla^2 r(\bbeta)\|= 2$. Therefore $H_{ii}>C$ w.p $\to 1$.

\section{Proof of Theorem \ref{thm: main_smooth}}\label{sec: heuristic}

\subsection{Proof of Part 1}
It follows from Lemma 3.3 of \cite{zou2025certified} (with $m=t=1$) that $\|\hbi-\tbin \| = o_P(\frac{1}{\sqrt{n}}\polylog(n))$. Thus
\begin{align*}\label{eq:newfluence-diff}
    \cIn-\cItrue =&~ \ell_0(\tbin)-\ell_0(\hbi)\\
    =&~\ell(y_0,\bx_0^{\top}\tbin)-\ell(y_0,\bx_0^{\top}\hbi)\\
    =&~\left(
    \int_0^1
    \nabla_{\bbeta}\ell(y_0,\bx_0^{\top}(\tbin+t(\tbin-\hbi)))dt
    \right)^{\top}(\tbin-\hbi)\\
    =&~\left(
    \int_0^1 \dl_0(\tbin+t(\tbin-\hbi))dt
    \right)\bx_0^{\top}(\tbin-\hbi)\\
    \le&~
    C
    \left(
    1+|y_0|^s+|\bx_0^{\top}\tbin|^s+|\bx_0^{\top}\hbi|^s\right)
    \bx_0^{\top}(\tbin-\hbi)\\
    \le&~
    4C\sqrt{C_X}
    \left(
    1+(C_y(n))^s+(3\|\hbi\|\sqrt{C_X\log(n)/p})^s\right)
    \|\tbin-\hbi\|\sqrt{\log(n)/p}\\
    =&~
    O_p\left(
    \frac{\polylog(n)}{n}
    \right)\numberthis
\end{align*}
with high probability. Here the first inequality follows from Assumption~\ref{assum:ld}. The second inequality holds with high probability by Assumption~\ref{assum:y} and by Lemma~\ref{lem:chi:sq:ind}, since $\bx_0\sim\cN(\mathbf{0},\bSigma)$ with $\lambda_{\max}(\bSigma)\le C_X/p$, and $\bx_0$ is independent of $\{(y_i,\bx_i):1\le i\le n\}$, and hence of $\hbi,\tbin$. In the last step we use the fact that $\|\hbi\|=O_{P}(\sqrt{p})$, and finally the above quoted bound on the error of the Newton step, i.e., $\|\hbi-\tbin \| = o_P(\frac{1}{\sqrt{n}}\polylog(n))$. 

\subsection{Proof of Part 2}

Next, note that by definition of $\cIn$, we have by a Taylor series expansion that
\begin{align*}\label{eq:newfluence-alt}
    \cIn-\frac{\dl_0(\hbbeta)\dl_i(\hbbeta)\bx_0^{\top}\bG^{-1}\bx_i}{1-H_{ii}}
    =&~
    \ell_0\left(\hbbeta+\frac{\dl_i(\hbbeta)\bG^{-1}\bx_i}{1-H_{ii}}\right)
    -\ell_0(\hbbeta)-\frac{\dl_0(\hbbeta)\dl_i(\hbbeta)\bx_0^{\top}\bG^{-1}\bx_i}{1-H_{ii}}\\
    =&~\left(\frac{\dl_i(\hbbeta)\bx_0^{\top}\bG^{-1}\bx_i}{1-H_{ii}}\right)^2\ddl_0(\xi)\\
    =&~O_p(\polylog(n))\cdot (\bx_0^{\top}\bG^{-1}\bx_i)^2
    =~O_P\left(
    \frac{\polylog(n)}{n}
    \right)\numberthis
\end{align*}
Here $\xi=t\hbbeta+(1-t)\tbin$ for some $t\in[0,1]$. To get the second last equality we use Lemma~\ref{lem:conc_normal_mat} to conclude that $\|\bG^{-1}\bx_i\|=O_P(1)$, Assumptions \ref{assum:ld} and \ref{assum:y} to arrive at $(1-H_{ii})^{-1}=O_P(1)$, $\dl_0(\hbbeta),\dl_i(\hbbeta),\ddl_i(\xi)=O_P(\polylog(n))$. 
To get the last equality we again use the conclusion that $\|\bG^{-1}\bx_i\|=O_P(1)$
, and finally that $\bx_0\sim \cN(\mathbf{0},\bSigma)$ with $\lambda_{\max}(\bSigma)=C_X/p$, independent of $\{(y_i,\bx_i):1\le i\le n\}$. Note also that $n/p\to\gamma_0$.

From \eqref{eq:newfluence-diff} and \eqref{eq:newfluence-alt} we then have
\begin{equation}\label{eq:true-inf-alt}
    \cItrue-\frac{\dl_0(\hbbeta)\dl_i(\hbbeta)\bx_0^{\top}\bG^{-1}\bx_i}{1-H_{ii}}
    =~O_P\left(
    \frac{\polylog(n)}{n}
    \right).
\end{equation}
By arguments identical to \eqref{eq:newfluence-alt} we have $(1-H_{ii})^{-1}=O_P(1)$, $\dl_0(\hbbeta),\dl_i(\hbbeta),\ddl_i(\xi)=O_P(\polylog(n))$, and $\|\bG^{-1}\bx_i\|=O_P(1)$, and thus 
\[
\frac{\dl_0(\hbbeta)\dl_i(\hbbeta)\bx_0^{\top}\bG^{-1}\bx_i}{1-H_{ii}}
=O_P\left(
    \frac{\polylog(n)}{\sqrt{n}}
    \right).
\]
This completes the proof of part 2. In fact, for many choices of loss functions, such as squared loss, logistic loss, or Poisson negative log likelihood the above quantity is in fact $\Theta_P\left(
    \frac{\polylog(n)}{\sqrt{n}}
    \right)$. We omit a more detailed analysis here.

\subsection{Proof of Part 3}
Since we are in the setup of generalized linear models, 
\[
\nabla_{\bbeta} \ell(\hbbeta, \bz_0)=\bx_0\dl_0(\hbbeta),
\quad 
\nabla_{\bbeta} \ell(\hbbeta, \bz_i)=\bx_i\dl_i(\hbbeta).
\]
Thus, definition of $\cIif$ and \eqref{eq:true-inf-alt} together imply that
\begin{align*}
   \cItrue-\frac{\cIif}{1-H_{ii}}
    =~O_P\left(
    \frac{\polylog(n)}{n}
    \right).
\end{align*}
from where the conclusion of Part 2 follows immediately since $0<1-H_{ii}\le 1$.

\section{Auxiliary Lemmata}
\begin{lemma}[Woodbury Inversion Formula]
\label{lem:woodberry}
    Suppose $\bA\in \RR^{n\times n}$ is nonsingular, and $\bM = \bA + \bU \bB \bV$, then 
    \[
        \bM^{-1} = \bA^{-1} - \bA^{-1}\bU(\bB^{-1} + \bV\bA^{-1}\bU)^{-1}\bV\bA^{-1}
    \]
    provided that all relevant inverse matrices exist.
\end{lemma}

\begin{lemma} [Lemma 6 of \citet{jalali2016}] \label{lem:chi:sq:ind}
    Let $\bz \sim N(0,\II_p)$, then
    \[
    \PP (\bz^{\top } \bz\geq p + pt ) \leq e^{-\frac{p}{2} (t- \log (1+t)) }
    \]
\end{lemma}
\begin{lemma}[Lemma 12 in \cite{rad2018scalable}]\label{lem:maxsingularvalue_0}
    $\bX \in \mathbb{R}^{p \times p}$ is composed of independently distributed $N(0, \bSigma)$ rows, with $\rho_{\max} = \sigma_{\max} (\Sigma)$, where $\Sigma \in \mathbb{R}^{p \times p}$. Then
    \[
    {\PP} (\|\bX^{\top} \bX\| \geq (\sqrt{n} + 3 \sqrt{p})^2 \rho_{\max}) \leq {\rm e}^{-p}. 
    \]
\end{lemma}
\begin{lemma}\label{lem:conc_normal_mat}
    Suppose $\bX_{n\times p}$ has iid rows $\bx_i\sim N(0,p^{-1}\II_p)$, then
    \begin{enumerate}
        \item $\PP (\max_{1\le i \le n}\|\bx_i\|\geq 2 ) \leq ne^{-p/2}$
        \item ${\PP} (\|\bX^{\top} \bX\| \geq (\sqrt{\gamma_0}+3)^2 ) \leq e^{-p}$
    \end{enumerate}
\end{lemma}
\begin{proof}
    \begin{enumerate}
        \item By Lemma \ref{lem:chi:sq:ind} and let $\bz = n^{-1/2}\bx_i$ we have $\bz\sim N(0,\II_p)$ so that
        \begin{align*}
            \PP (\|\bx_i\|\geq 2 ) &= \PP(\bx_i^\top\bx_i\geq 4)\\
            &\leq \PP(\bz^\top\bz\geq 4p)\\
            &\leq e^{-\frac{p}{2}(3-\log(4))}\leq e^{-p/2}
        \end{align*}
        The rest follows from a union bound over all $i$.
        \item It is a direct application of \ref{lem:maxsingularvalue_0} with $\rho_{\max}=p^{-1}$.
    \end{enumerate}
\end{proof}

\end{document}